\documentclass[11pt]{article}
\usepackage{amsmath}
\usepackage{amssymb}
\usepackage{amsthm}
\usepackage{mathptmx}
\usepackage{fullpage}
\usepackage{aliascnt}
\usepackage[hidelinks]{hyperref}
\usepackage{cleveref}

\newcommand{\R}{\mathbb{R}}
\newcommand{\Z}{\mathbb{Z}}
\newcommand{\Max}{\operatorname{MAX}}
\newcommand{\ReLU}{\operatorname{ReLU}}

\newtheorem{theorem}{Theorem}[section]
\newaliascnt{lemma}{theorem}
\newtheorem{lemma}[lemma]{Lemma}
\aliascntresetthe{lemma}
\newaliascnt{corollary}{theorem}
\newtheorem{corollary}[corollary]{Corollary}
\aliascntresetthe{corollary}
\theoremstyle{definition}
\newaliascnt{definition}{theorem}
\newtheorem{definition}[definition]{Definition}
\aliascntresetthe{definition}
\newaliascnt{remark}{theorem}
\newtheorem{remark}[remark]{Remark}
\aliascntresetthe{remark}

\title{\textbf{A Hand-Checkable Proof That Two Hidden ReLU Layers Compute the Maximum of Six Numbers}}
\author{Dimitrios Myrisiotis\\Great Bay University\\China}
\date{}

\begin{document}
\maketitle

\begin{abstract}
Exactly computing the maximum function is a standard test case for studying depth in ReLU networks.
Two hidden layers are known to suffice for up to twelve inputs through computer-assisted constructions.
For six real inputs, we give an explicit hexagon identity whose local structure yields a self-contained analytical proof of this depth bound.
The identity was found by computer-assisted search; we prove it through explicit cancellations that can be checked entirely by hand, without executing a verification program.
The identity also yields an explicit network with hidden widths $17$ and $41$, zero biases, and rational weights.
\end{abstract}

\section{Introduction}
\label{sec:introduction}

A ReLU network alternates affine maps with layers of gates, and each gate applies the rectifier $\ReLU(t)=\max(t,0)$ to one number.
Arora, Basu, Mianjy, and Mukherjee~\cite{ABMM2018} showed that every continuous piecewise linear function on $\R^n$ is computed exactly by a ReLU network with ${\lceil\log_2(n+1)\rceil}$ hidden layers.
A basic question about such networks is how many hidden layers are needed, and the maximum function is the standard test case.
Write $\Max_n(x)=\max\{x_0,\ldots,x_{n-1}\}$ for the largest of $n$ real numbers.
Since $\max(s,t)=t+\ReLU(s-t)$ and $t=\ReLU(t)-\ReLU(-t)$, the maximum of two numbers takes one hidden layer.
Pairing the inputs and taking maxima of pairs repeatedly therefore computes $\Max_n$ with $\lceil\log_2 n\rceil$ hidden layers.
Hertrich, Basu, Di Summa, and Skutella~\cite{HBDSS23} conjectured that $\lceil\log_2 n\rceil$ hidden layers are necessary to compute $\Max_n$.
Bakaev, Brunck, Hertrich, Stade, and Yehudayoff~\cite{BBHSY2026} disproved the conjecture.
A key step of their argument is a network with two hidden layers that computes $\Max_5$, and the proof that it computes $\Max_5$ is short and by hand.

Beyond five inputs, the proofs known to us are computer-assisted.
Rue\ss{} et~al.~\cite{RuessEtAl2026} proved that two hidden layers are necessary and sufficient to compute $\Max_n$ exactly for every $n$ with $3\le n\le12$.
Wang and Basu~\cite{WangBasu2026} independently constructed two-hidden-layer representations of $\Max_5$ through $\Max_8$.
Both works obtain their identities by computer-assisted searches over finite rational linear systems.
Rue\ss{} et~al. also report independent verification by an exact-arithmetic program.
The algebraic reductions in both papers give finite checks for the proposed identities after restricting to inputs whose coordinates are in sorted order.
Our contribution for six inputs is an explicit identity organized by the rotations and reflections of a hexagon, together with a complete analytical verification.
The proof explains how local relations among pairwise maxima produce the cancellations needed to recover the maximum.
Every step is written out for hand checking.

\subsection{Our Results}
\label{sec:results}

Place the indices $0,1,\ldots,5$ on the hexagon $0-1-2-3-4-5-0$ and read them modulo $6$.
Write $\mu_i=\max(x_i,x_{i+1})$ for the maximum over the edge from $i$ to $i+1$.
The identity below uses twelve of the fifteen pairwise maxima $\max(x_a,x_b)$: those over the six edges and over the six short diagonals $\{i,i+2\}$.

\begin{theorem}[restated as \Cref{thm:hexagon}]
\label{ithm:hexagon}
There are thirty-nine linear forms $F_1,\ldots,F_{39}$ in the maxima over the edges and the short diagonals of the hexagon, each with integer coefficients summing to zero, and signs $s_1,\ldots,s_{39}$, each equal to $1$ or $-1$, such that for every $x\in\R^6$,
\[
12\Max_6(x)=2\sum_{i=0}^{5}\mu_i(x)+\sum_{k=1}^{39}s_k\left|F_k(x)\right|.
\]
\end{theorem}

The forms are listed explicitly in \Cref{def:hexagon-forms}, in three families: six doubled differences of edges at distance two, nine pairings of edges, and twenty-four forms through the short diagonals.
Each form becomes the input of one gate in the second hidden layer of the network below, so we call the forms gate forms and the twenty-four forms of the third family diagonal gates.
Twenty-one of the forms carry the sign $+1$ and eighteen carry $-1$.
Rotating or reflecting the hexagon permutes the absolute-value terms, preserves each term's plus or minus coefficient, and leaves the sum of the edge maxima unchanged (\Cref{lem:hexagon-rotation} and \Cref{rem:hexagon-reflection}).
Thus the displayed expression respects all twelve symmetries of the hexagon term by term.

Each pairwise maximum takes one hidden layer, as noted above.
An absolute value $|t|=\ReLU(t)+\ReLU(-t)$ of a linear combination of pairwise maxima then takes a second.
So \Cref{ithm:hexagon} gives a network with two hidden layers, and counting the gates carefully gives the following.

\begin{corollary}[restated as \Cref{cor:hexagon-network}]
\label{icor:hexagon-network}
There is a ReLU network with two hidden layers, of widths $17$ and $41$, and with all biases zero, that computes $\Max_6$ on all of $\R^6$.
Every weight of the network is an integer multiple of $1/6$.
\end{corollary}

Repeating inputs, as in $\Max_5(x_0,\ldots,x_4)=\Max_6(x_0,\ldots,x_4,x_4)$, turns the same network into one for $\Max_n$ for every $1\le n\le6$, after merging the first-layer weights of the repeated inputs.
For $n=5$ this gives another network with two hidden layers for the maximum of five numbers, whose existence Bakaev et~al.~\cite{BBHSY2026} proved by a different argument.

\subsection{Our Techniques}
\label{sec:techniques}

The proof works with the identity directly and never lists the orderings of the six inputs.
Write $M=\Max_6(x)$ for the largest value.
Since the identity is invariant under rotations, we may rotate the hexagon so that a coordinate equal to $M$ sits at the vertex $0$.
Then the maxima over the two edges and the two short diagonals that end at the vertex $0$ equal $M$.
Each of the other four edge maxima $\mu_i$, for $i=1,\ldots,4$, equals $M-p_i$ for a gap $p_i\ge0$.
Every form without a diagonal becomes a linear expression in these four gaps.

A diagonal maximum $\max(x_{v-1},x_{v+1})$ is not determined by the edge maxima alone.
On the triangle $\{v-1,v,v+1\}$, however, it equals the maximum over one of the two edges at $v$, unless $x_v$ is the largest of the three.
If $x_v$ is the largest of the three, then the maxima over the two edges at $v$ are equal, and the absolute values of the four diagonal gates through that diagonal cancel.
The triangles $\{v-1,v,v+1\}$ that avoid the vertex $0$ are those with $v=2,3,4$, and each of them needs three cases, one for each vertex that may be largest.
In the two cases where $x_v$ is not the largest, each diagonal gate becomes a linear expression in the gaps, and its absolute value enters the sum.
At the two neighbors of the vertex $0$ the diagonal maximum itself equals $M$, so the diagonal gates there become such expressions without any case distinction.

At the vertices $0$ and $3$ the signed sum of the absolute values of the four diagonal gates vanishes on its own.
At each of the vertices $2$ and $4$, the absolute values of the diagonal gates cancel two pairing terms.
At the vertices $1$ and $5$ they cancel what remains of the pairing terms and part of the distance-two differences.
The sum of all thirty-nine signed absolute values becomes $2(p_1+p_2+p_3+p_4)$, which is the value the identity requires.

An exact computer search produced the hexagon identity.
Its verification is a separate analytical argument: the rotation reduction, the triangle relations, and all subsequent cancellations are established explicitly in the paper.
Thus the search supplies the identity, and the proof explains why it holds for every real input, including ties.

\subsection{Other Related Work}
\label{sec:related}

The identities of Rue\ss{} et~al.~\cite{RuessEtAl2026} and of Wang and Basu~\cite{WangBasu2026} are linear combinations of sums of atoms of the form $\max(A,B)$.
Here $A$ and $B$ are sums of pairwise maxima, and each sum of atoms collects the expressions obtained from one atom by permuting the input coordinates.
Each term $|F_k|$ of \Cref{ithm:hexagon} is a combination of such an atom and pairwise maxima, since $|A-B|=2\max(A,B)-A-B$.
Our identity has two features relevant to its proof: the twelve hexagon symmetries permute its absolute-value terms while preserving their coefficients, and it avoids the three long diagonals $\max(x_i,x_{i+3})$.
These features organize the verification around local relations on the hexagon.
Rotations place a largest input at one vertex, and every short diagonal $\{v-1,v+1\}$ lies in the triangle $\{v-1,v,v+1\}$ of consecutive vertices, where the three cases described in \Cref{sec:techniques} apply.
The resulting cancellations are displayed directly in the proof, so verifying this identity requires no enumeration of the orderings of all six inputs.
The depth bound of Bakaev et~al.~\cite{BBHSY2026} for all continuous piecewise linear functions is built on their representation of $\Max_5$.

\subsection{Discussion}
\label{sec:discussion}

The statement that two hidden layers compute $\Max_6$ was known~\cite{RuessEtAl2026,WangBasu2026}.
The hexagon identity gives a structural explanation of this known depth bound for six inputs.
Its terms are arranged so that local triangle relations expose the cancellations, and the complete verification can be followed by hand.
The earlier constructions cover more inputs: Rue\ss{} et~al. reach twelve, and Wang and Basu reach eight, using symmetric orbit sums and finite exact linear algebra~\cite{RuessEtAl2026,WangBasu2026}.
The identity also gives an explicit network of widths $17$ and $41$, with zero biases and weights that are integer multiples of $1/6$.

\subsection{Open Problems}
\label{sec:open}

\paragraph{Seven or more inputs.}

Two hidden layers compute $\Max_n$ for $n\le12$~\cite{RuessEtAl2026}.
The constructions for seven through twelve inputs in that work are computer-assisted.
A short proof for seven inputs that a reader can check by hand is not known to us.
One possible approach is to seek an identity on a heptagon using only edges and diagonals that join the neighbors of a vertex.
Rotating a largest coordinate to one vertex would then permit the same triangle reductions.
The remaining challenge is to find an identity whose terms admit comparably short cancellations.

\paragraph{A uniform explanation.}

Our proof evaluates every term of the identity after one rotation.
A conceptual construction of identities using maxima over polygon edges and diagonals, for each $3\le n\le12$, would explain the depth bound of~\cite{RuessEtAl2026}.

\paragraph{Smallest networks.}

The smallest widths of a two-hidden-layer network computing $\Max_6$ exactly are not known.
The network of \Cref{icor:hexagon-network} gives the upper bounds $17$ and $41$.

\phantomsection\addcontentsline{toc}{subsection}{AI Disclosure}
\subsection*{AI Disclosure}

Generative AI tools, including large language models (LLMs), were used to assist with proof development, literature searches, and manuscript preparation.
The author formulated the problem, guided the development, and checked the mathematical arguments and references.
The author takes responsibility for all claims and any errors.

\subsection{Paper Organization}
\label{sec:organization}

\Cref{sec:preliminaries} fixes the notions of a network and of the pairwise maxima on the hexagon.
\Cref{sec:hexagon} proves the identity of \Cref{ithm:hexagon}, and \Cref{sec:network} builds the network of \Cref{icor:hexagon-network}.

\section{Preliminaries}
\label{sec:preliminaries}

This section fixes what a network with two hidden layers is, and the pairwise maxima on the hexagon that both later sections use.
Throughout, the coordinates of $x\in\R^6$ are indexed by $0,\ldots,5$.

A ReLU network applies an affine map, then the rectifier $\ReLU(t)=\max(t,0)$ in every coordinate, and repeats; the last affine map gives the output.
The width of a hidden layer is its number of rectifiers, called gates.

\begin{definition}[ReLU network with two hidden layers]
\label{def:network}
Let $\ReLU$ act coordinatewise on vectors.
A ReLU network with two hidden layers and architecture $n\to u\to v\to1$ computes
\[
x\mapsto W_3\ReLU\left(W_2\ReLU\left(W_1x+b_1\right)+b_2\right)+b_3,
\]
with weight matrices $W_1\in\R^{u\times n}$, $W_2\in\R^{v\times u}$ and $W_3\in\R^{1\times v}$, and biases $b_1\in\R^u$, $b_2\in\R^v$ and $b_3\in\R$.
Its hidden widths are $u$ and $v$, and it is \emph{zero-bias} if $b_1$, $b_2$ and $b_3$ are zero.
It computes a function $f$ on $\R^n$ if its output equals $f(x)$ for every $x\in\R^n$.
\end{definition}

In this form the second layer reads only the first, and the output reads only the second, so there are no skip connections.
The identity $t=\ReLU(t)-\ReLU(-t)$ lets a pair of second-layer gates pass a linear form in the outputs of the first layer through to the output.

For indices $a,b\in\{0,\ldots,5\}$ write $m_{ab}=\max(x_a,x_b)$ for the pairwise maximum, a function of $x$.
Place the indices on the hexagon $0-1-2-3-4-5-0$ and read every index modulo $6$.
The maxima over its edges and over its short diagonals are
\[
\mu_i=m_{i,i+1}\qquad\text{and}\qquad\delta_v=m_{v-1,v+1}\qquad(i,v\in\{0,\ldots,5\}).
\]
So $\mu_i$ is the maximum over the edge from $i$ to $i+1$, and $\delta_v$ is the maximum over the short diagonal that skips the vertex $v$.
Explicitly,
\[
\begin{aligned}
(\mu_0,\ldots,\mu_5)&=(m_{01},m_{12},m_{23},m_{34},m_{45},m_{05}),\\
(\delta_0,\ldots,\delta_5)&=(m_{15},m_{02},m_{13},m_{24},m_{35},m_{04}).
\end{aligned}
\]
The maxima $m_{03}$, $m_{14}$ and $m_{25}$ over the three long diagonals never occur in the identity of \Cref{ithm:hexagon}.
Where no confusion can arise, we name an edge by its maximum.
So the two edges at the vertex $v$ are $\mu_{v-1}$ and $\mu_v$, and the edges $\mu_i$ and $\mu_{i+3}$ are opposite.

\section{The Hexagon Identity}
\label{sec:hexagon}

This section defines the $39$ forms of \Cref{ithm:hexagon} in the maxima $\mu_i$ and $\delta_v$ of \Cref{sec:preliminaries}, and proves the identity.
The forms come in three families, each described by one formula in the indices modulo $6$.

\begin{definition}[Hexagon gate forms]
\label{def:hexagon-forms}
The hexagon gate forms are the following $39$ linear forms in the edge and diagonal maxima, in three families.
The name comes from the network of \Cref{sec:network}, whose second hidden layer has one gate computing $\ReLU(F(x))$ for each form $F$.
\begin{enumerate}
\item Distance-two edge differences.
For $j\in\{0,\ldots,5\}$, the edges $\mu_j$ and $\mu_{j+2}$ are separated by one edge, and
\[
\Theta_j=2(\mu_j-\mu_{j+2}).
\]
\item Pairings.
For $k\in\{0,1,2\}$, remove the opposite pair $\mu_k,\mu_{k+3}$.
The four remaining edges $\mu_{k+1}$, $\mu_{k+2}$, $\mu_{k+4}$ and $\mu_{k+5}$ split into two pairs in exactly three ways, and each splitting gives one form:
\[
\begin{aligned}
\Pi_k&=\mu_{k+1}+\mu_{k+2}-\mu_{k+4}-\mu_{k+5},\\
\Omega_k&=\mu_{k+2}+\mu_{k+5}-\mu_{k+1}-\mu_{k+4},\\
\Gamma_k&=\mu_{k+1}+\mu_{k+5}-\mu_{k+2}-\mu_{k+4}.
\end{aligned}
\]
In $\Pi_k$ the path of two edges through the vertex $k+2$ is compared with the path through the vertex $k+5$.
The form $\Omega_k$ compares the opposite pair $\mu_{k+2},\mu_{k+5}$ with the opposite pair $\mu_{k+1},\mu_{k+4}$.
Finally, $\Gamma_k$ compares the two edges adjacent to $\mu_k$ with the two edges adjacent to $\mu_{k+3}$.
\item Diagonal gates.
For $v\in\{0,\ldots,5\}$,
\[
\begin{aligned}
\Phi_v^+&=\delta_v-\mu_v-\mu_{v+1}+\mu_{v+3}, &\qquad \Psi_v^+&=\delta_v-\mu_{v-1}-\mu_{v+1}+\mu_{v+3},\\
\Phi_v^-&=\delta_v-\mu_{v-1}-\mu_{v-2}+\mu_{v+2}, &\qquad \Psi_v^-&=\delta_v-\mu_v-\mu_{v-2}+\mu_{v+2}.
\end{aligned}
\]
In words, $\Phi_v^+$ subtracts from $\delta_v$ the edge $\mu_v$ at $v$ and the next edge $\mu_{v+1}$, and adds the edge $\mu_{v+3}$ opposite to $\mu_v$.
Its mirror image through the vertex $v$ is $\Phi_v^-$, which subtracts $\mu_{v-1}$ and $\mu_{v-2}$, and adds $\mu_{v+2}$, the edge opposite to $\mu_{v-1}$.
The form $\Psi_v^+$ is $\Phi_v^+$ with the edge $\mu_v$ replaced by the other edge $\mu_{v-1}$ at $v$.
Likewise, $\Psi_v^-$ is $\Phi_v^-$ with $\mu_{v-1}$ replaced by $\mu_v$.
\end{enumerate}
The families contain $6$, $3\cdot3=9$ and $6\cdot4=24$ forms.
The $21$ forms $\Theta_j$, $\Omega_k$, $\Phi_v^+$, $\Phi_v^-$ are the positive gate forms, and the $18$ forms $\Pi_k$, $\Gamma_k$, $\Psi_v^+$, $\Psi_v^-$ are the negative gate forms.
\end{definition}

As linear forms in the twelve maxima, these $39$ forms are pairwise distinct, and no one of them is the negative of another.
The nonzero coefficients identify a form's family even when the form is negated: only the forms $\Theta_j$ have coefficients $\pm2$, and only the diagonal gates contain a diagonal maximum.
If $\Theta_{j'}=\pm\Theta_j$, then the coefficients $\pm2$ sit on the same two edges, so $\{j',j'+2\}=\{j,j+2\}$ modulo $6$.
Here $j'=j+2$ together with $j'+2=j$ would give $4\equiv0$ modulo $6$, so $j'=j$, and then the sign is $+$.
The edge supports of the pairings for $k=0,1,2$ are the distinct sets $\{1,2,4,5\}$, $\{0,2,3,5\}$ and $\{0,1,3,4\}$.
The three pairings with one support split it into two pairs in three different ways, and a form and its negative split their support in the same way, so no pairing equals another or its negative.

Every diagonal gate has the coefficient $+1$ at its diagonal, so no diagonal gate is the negative of another, and diagonal gates at different vertices contain different diagonals.
At one vertex $v$, the edge with coefficient $+1$ is $\mu_{v+3}$ in $\Phi_v^+$ and $\Psi_v^+$ and is $\mu_{v+2}$ in $\Phi_v^-$ and $\Psi_v^-$.
Of each of these two pairs, exactly one form contains $-\mu_v$, namely $\Phi_v^+$ and $\Psi_v^-$, so the four diagonal gates at $v$ are distinct.

Every hexagon gate form is a combination of pairwise maxima with integer coefficients summing to zero.
Indeed, $\Theta_j$ has coefficients $2,-2$, and every other form has two coefficients $+1$ and two coefficients $-1$.
We group the absolute values of the forms as follows:
\[
\begin{aligned}
S_\Theta&=\sum_{j=0}^{5}\left|\Theta_j\right|,\qquad
Y_k=\left|\Omega_k\right|-\left|\Pi_k\right|-\left|\Gamma_k\right|,\\
Z_v&=\left|\Phi_v^+\right|+\left|\Phi_v^-\right|-\left|\Psi_v^+\right|-\left|\Psi_v^-\right|.
\end{aligned}
\]
Here $Y_k$ collects the three pairings for the opposite pair $\mu_k,\mu_{k+3}$, and $Z_v$ collects the four diagonal gates that contain $\delta_v$.
We call $Y_k$ and $Z_v$ aggregates.

\begin{theorem}[restatement of \Cref{ithm:hexagon}]
\label{thm:hexagon}
For every $x\in\R^6$,
\begin{equation}\label{eq:hexagon-abs}
12\Max_6(x)=2\sum_{i=0}^{5}\mu_i+S_\Theta+\sum_{k=0}^{2}Y_k+\sum_{v=0}^{5}Z_v,
\end{equation}
where every $\mu_i$ and every form is evaluated at $x$.
Equivalently, for every $x\in\R^6$,
\begin{equation}\label{eq:hexagon-relu}
\begin{split}
\Max_6(x)={}&\frac13\left(\mu_1+\mu_2+\mu_3\right)+\frac16\sum_{j=0}^{5}\ReLU(\Theta_j)\\
&+\frac16\sum_{k=0}^{2}\left(\ReLU(\Omega_k)-\ReLU(\Pi_k)-\ReLU(\Gamma_k)\right)\\
&+\frac16\sum_{v=0}^{5}\left(\ReLU(\Phi_v^+)+\ReLU(\Phi_v^-)\right)\\
&-\frac16\sum_{v=0}^{5}\left(\ReLU(\Psi_v^+)+\ReLU(\Psi_v^-)\right).
\end{split}
\end{equation}
\end{theorem}

In \Cref{eq:hexagon-relu} every positive gate form carries the coefficient $1/6$ and every negative gate form carries $-1/6$.
The proof of \Cref{eq:hexagon-abs} places a largest coordinate at the vertex $0$ by a rotation, and then evaluates the four groups of terms directly.
Two lemmas prepare it.

\begin{lemma}[Rotation]
\label{lem:hexagon-rotation}
Let $t\in\{0,\ldots,5\}$, and for $x\in\R^6$ define $y\in\R^6$ by $y_i=x_{i+t}$.
Then there is a bijection $\sigma$ of the $39$ hexagon gate forms onto themselves that maps positive gate forms to positive gate forms and negative gate forms to negative gate forms, and such that the following holds.
For each form $F$, either $F(y)=\sigma(F)(x)$ for every $x\in\R^6$ or $F(y)=-\sigma(F)(x)$ for every $x\in\R^6$.
Moreover, for every $x\in\R^6$, $\sum_i\mu_i(y)=\sum_i\mu_i(x)$, and each side of \Cref{eq:hexagon-abs} takes the same value at $y$ as at $x$.
\end{lemma}

\begin{proof}
From the definitions, $\mu_i(y)=\max(x_{i+t},x_{i+t+1})=\mu_{i+t}(x)$ and $\delta_v(y)=\max(x_{v+t-1},x_{v+t+1})=\delta_{v+t}(x)$.
Every form in \Cref{def:hexagon-forms} is written by one formula in the indices modulo $6$.
For the pairings we let that formula define $\Pi_k$, $\Omega_k$, $\Gamma_k$ and $Y_k$ for every $k$ modulo $6$.
Replacing $k$ by $k+3$ in the formulas gives
\[
\begin{aligned}
\Pi_{k+3}&=\mu_{k+4}+\mu_{k+5}-\mu_{k+1}-\mu_{k+2}=-\Pi_k,\\
\Omega_{k+3}&=\mu_{k+5}+\mu_{k+2}-\mu_{k+4}-\mu_{k+1}=\Omega_k,
\end{aligned}
\]
and $\Gamma_{k+3}=\mu_{k+4}+\mu_{k+2}-\mu_{k+5}-\mu_{k+1}=-\Gamma_k$.
Hence $Y_{k+3}=Y_k$, so $Y_k$ depends only on $k$ modulo $3$.
In each formula every index is the index of the form plus a constant, so substituting $\mu_i(y)=\mu_{i+t}(x)$ and $\delta_v(y)=\delta_{v+t}(x)$ adds $t$ to the index of the form.
For example, $\Phi_v^+(y)=\delta_{v+t}-\mu_{v+t}-\mu_{v+t+1}+\mu_{v+t+3}=\Phi_{v+t}^+(x)$, where the middle expression is evaluated at $x$.
So $\Theta_j(y)=\Theta_{j+t}(x)$, $\Pi_k(y)=\Pi_{k+t}(x)$, $\Omega_k(y)=\Omega_{k+t}(x)$ and $\Gamma_k(y)=\Gamma_{k+t}(x)$, and each of the four diagonal gates at $v$, evaluated at $y$, equals the diagonal gate of the same name at $v+t$, evaluated at $x$.
Taking absolute values gives $Y_k(y)=Y_{k+t}(x)$ and $Z_v(y)=Z_{v+t}(x)$.
As $j$ and $v$ run over all residues modulo $6$, so do $j+t$ and $v+t$, and therefore $S_\Theta(y)=S_\Theta(x)$ and $\sum_vZ_v(y)=\sum_vZ_v(x)$.
As $k$ runs over $0,1,2$, the residues of $k+t$ modulo $3$ are again $0,1,2$ in some order, so $\sum_{k=0}^{2}Y_k(y)=\sum_{k=0}^{2}Y_k(x)$.
By $\Pi_{k+3}=-\Pi_k$, $\Omega_{k+3}=\Omega_k$ and $\Gamma_{k+3}=-\Gamma_k$, the pairings $\Pi_{k+t}$, $\Omega_{k+t}$ and $\Gamma_{k+t}$ are plus or minus $\Pi_{k'}$, $\Omega_{k'}$ and $\Gamma_{k'}$, with signs depending only on $k$ and $t$, where $k'\in\{0,1,2\}$ and $k'\equiv k+t$ modulo $3$.

Let $\sigma$ map $\Theta_j$ to $\Theta_{j+t}$, each diagonal gate at $v$ to the diagonal gate of the same name at $v+t$, and $\Pi_k$, $\Omega_k$, $\Gamma_k$ to $\Pi_{k'}$, $\Omega_{k'}$, $\Gamma_{k'}$.
Since $j\mapsto j+t$ and $v\mapsto v+t$ permute the residues modulo $6$ and $k\mapsto k'$ permutes $\{0,1,2\}$, the map $\sigma$ is a bijection of the $39$ forms.
It keeps the name of each form, so it maps positive gate forms to positive gate forms and negative gate forms to negative gate forms.
The equalities above hold for every $x$, so for each form $F$ either $F(y)=\sigma(F)(x)$ for every $x$ or $F(y)=-\sigma(F)(x)$ for every $x$.
Finally $\sum_i\mu_i(y)=\sum_i\mu_{i+t}(x)=\sum_i\mu_i(x)$, since $i\mapsto i+t$ permutes the residues modulo $6$, and $\Max_6(y)=\Max_6(x)$ because $y$ lists the coordinates of $x$ in another order.
\end{proof}

The second lemma relates the four diagonal gates at a vertex, and lists the three possibilities for the largest coordinate of the triangle around that vertex.

\begin{lemma}[Diagonal gates and triangles]
\label{lem:hexagon-diagonal}
Fix $v\in\{0,\ldots,5\}$ and $x\in\R^6$.
\begin{enumerate}
\item $\Phi_v^+-\Psi_v^+=\mu_{v-1}-\mu_v=\Psi_v^--\Phi_v^-$.
In particular, if $\mu_{v-1}=\mu_v$, then $\Phi_v^+=\Psi_v^+$, $\Phi_v^-=\Psi_v^-$ and $Z_v=0$.
\item At least one of the following holds:
(a) $x_v\ge x_{v-1}$ and $x_v\ge x_{v+1}$, and then $\mu_{v-1}=\mu_v=x_v$;
(b) $x_{v-1}\ge x_v$ and $x_{v-1}\ge x_{v+1}$, and then $\delta_v=\mu_{v-1}=x_{v-1}\ge\mu_v$;
(c) $x_{v+1}\ge x_{v-1}$ and $x_{v+1}\ge x_v$, and then $\delta_v=\mu_v=x_{v+1}\ge\mu_{v-1}$.
\end{enumerate}
\end{lemma}

\begin{proof}
For the first part, subtract the defining formulas.
The forms $\Phi_v^+$ and $\Psi_v^+$ share the terms $\delta_v-\mu_{v+1}+\mu_{v+3}$, so $\Phi_v^+-\Psi_v^+=-\mu_v+\mu_{v-1}$.
The forms $\Psi_v^-$ and $\Phi_v^-$ share the terms $\delta_v-\mu_{v-2}+\mu_{v+2}$, so $\Psi_v^--\Phi_v^-=-\mu_v+\mu_{v-1}$.
If $\mu_{v-1}=\mu_v$, both differences vanish, and the four absolute values in $Z_v$ cancel in pairs.

For the second part, one of the three numbers $x_{v-1},x_v,x_{v+1}$ is at least the other two, so at least one of (a), (b), (c) holds; with ties, several may hold.
In case (a), $\mu_{v-1}=\max(x_{v-1},x_v)=x_v$ and $\mu_v=\max(x_v,x_{v+1})=x_v$.
In case (b), $\delta_v=\max(x_{v-1},x_{v+1})=x_{v-1}$ and $\mu_{v-1}=\max(x_{v-1},x_v)=x_{v-1}$, while $\mu_v=\max(x_v,x_{v+1})\le x_{v-1}$.
Case (c) is case (b) with the roles of $v-1$ and $v+1$ exchanged: $\delta_v=x_{v+1}$, $\mu_v=x_{v+1}$ and $\mu_{v-1}=\max(x_{v-1},x_v)\le x_{v+1}$.
\end{proof}

\begin{proof}[Proof of \Cref{thm:hexagon}]
\emph{Reduction.}
Let $t$ be an index with $x_t=\Max_6(x)$.
By \Cref{lem:hexagon-rotation}, applied with this $t$, \Cref{eq:hexagon-abs} holds at $x$ if and only if it holds at the rotated point $y$, and $y_0=x_t$ is a largest coordinate of $y$.
We may therefore assume $x_0=M$, where $M=\Max_6(x)$; ties among the coordinates are allowed throughout.

\emph{Values at the top vertex.}
Since $x_0=M$ and every coordinate is at most $M$, the two edges at the vertex $0$ satisfy $\mu_0=\max(x_0,x_1)=M$ and $\mu_5=\max(x_5,x_0)=M$.
Likewise $\delta_1=\max(x_0,x_2)=M$ and $\delta_5=\max(x_4,x_0)=M$.
For $i\in\{1,2,3,4\}$ put $p_i=M-\mu_i$.
Each $\mu_i$ is one of the coordinates, so $p_i\ge0$.
It is convenient to write $p_0=p_5=0$ as well, so that $\mu_i=M-p_i$ for all $i$.
The target value of the sum of the three groups of absolute values is
\[
\begin{aligned}
12M-2\sum_{i=0}^{5}\mu_i&=12M-2\left(6M-\left(p_1+p_2+p_3+p_4\right)\right)\\
&=2\left(p_1+p_2+p_3+p_4\right).
\end{aligned}
\]
Thus \Cref{eq:hexagon-abs} is equivalent to
\begin{equation}\label{eq:hexagon-target}
S_\Theta+\sum_{k=0}^{2}Y_k+\sum_{v=0}^{5}Z_v=2\left(p_1+p_2+p_3+p_4\right).
\end{equation}
If a form has coefficient sum zero in the edges alone, then substituting $\mu_i=M-p_i$ cancels $M$ and replaces each $\mu_i$ by $-p_i$.
This applies to $\Theta_j$, $\Pi_k$, $\Omega_k$ and $\Gamma_k$.

\emph{The distance-two differences.}
The six forms are $\Theta_j=2(p_{j+2}-p_j)$, with $p_0=p_5=0$.
For $j=0,1,2,3,4,5$ they are $2p_2$, $2(p_3-p_1)$, $2(p_4-p_2)$, $-2p_3$, $-2p_4$ and $2p_1$.
Since every $p_i\ge0$, the absolute values are $2p_2$, $2\left|p_1-p_3\right|$, $2\left|p_2-p_4\right|$, $2p_3$, $2p_4$ and $2p_1$.
Therefore
\begin{equation}\label{eq:hexagon-theta}
S_\Theta=2\left(p_1+p_2+p_3+p_4\right)+2\left|p_1-p_3\right|+2\left|p_2-p_4\right|.
\end{equation}

\emph{The pairings.}
For $k=0$ the forms are $\Pi_0=\mu_1+\mu_2-\mu_4-\mu_5$, $\Omega_0=\mu_2+\mu_5-\mu_1-\mu_4$ and $\Gamma_0=\mu_1+\mu_5-\mu_2-\mu_4$.
Replacing each $\mu_i$ by $-p_i$ and using $p_5=0$ gives $\Pi_0=p_4-p_1-p_2$, $\Omega_0=p_1+p_4-p_2$ and $\Gamma_0=p_2+p_4-p_1$.
Hence
\[
Y_0=\left|p_1+p_4-p_2\right|-\left|p_1+p_2-p_4\right|-\left|p_2+p_4-p_1\right|.
\]
For $k=1$ the forms are $\Pi_1=\mu_2+\mu_3-\mu_5-\mu_0$, $\Omega_1=\mu_3+\mu_0-\mu_2-\mu_5$ and $\Gamma_1=\mu_2+\mu_0-\mu_3-\mu_5$.
With $p_0=p_5=0$ they become $\Pi_1=-(p_2+p_3)$, $\Omega_1=p_2-p_3$ and $\Gamma_1=p_3-p_2$.
Since $p_2,p_3\ge0$, we have $\left|\Pi_1\right|=p_2+p_3$, and $\left|\Omega_1\right|=\left|\Gamma_1\right|$.
Hence
\[
Y_1=-(p_2+p_3).
\]
For $k=2$ the forms are $\Pi_2=\mu_3+\mu_4-\mu_0-\mu_1$, $\Omega_2=\mu_4+\mu_1-\mu_3-\mu_0$ and $\Gamma_2=\mu_3+\mu_1-\mu_4-\mu_0$.
With $p_0=0$ they become $\Pi_2=p_1-p_3-p_4$, $\Omega_2=p_3-p_1-p_4$ and $\Gamma_2=p_4-p_1-p_3$.
Hence
\[
Y_2=\left|p_1+p_4-p_3\right|-\left|p_1+p_3-p_4\right|-\left|p_3+p_4-p_1\right|.
\]

\emph{The diagonal gates at the vertices $0$, $1$ and $5$.}
At $v=0$ the two edges are $\mu_5=\mu_0=M$, so $Z_0=0$ by \Cref{lem:hexagon-diagonal}.
At $v=1$ we use $\delta_1=M$ and $\mu_i=M-p_i$:
\[
\begin{aligned}
\Phi_1^+&=\delta_1-\mu_1-\mu_2+\mu_4=p_1+p_2-p_4,\\
\Phi_1^-&=\delta_1-\mu_0-\mu_5+\mu_3=-p_3,\\
\Psi_1^+&=\delta_1-\mu_0-\mu_2+\mu_4=p_2-p_4,\\
\Psi_1^-&=\delta_1-\mu_1-\mu_5+\mu_3=p_1-p_3.
\end{aligned}
\]
Since $p_3\ge0$, $\left|\Phi_1^-\right|=p_3$, and therefore
\[
Z_1=\left|p_1+p_2-p_4\right|+p_3-\left|p_2-p_4\right|-\left|p_1-p_3\right|.
\]
At $v=5$ we use $\delta_5=M$ in the same way:
\[
\begin{aligned}
\Phi_5^+&=\delta_5-\mu_5-\mu_0+\mu_2=-p_2,\\
\Phi_5^-&=\delta_5-\mu_4-\mu_3+\mu_1=p_3+p_4-p_1,\\
\Psi_5^+&=\delta_5-\mu_4-\mu_0+\mu_2=p_4-p_2,\\
\Psi_5^-&=\delta_5-\mu_5-\mu_3+\mu_1=p_3-p_1.
\end{aligned}
\]
Since $p_2\ge0$, $\left|\Phi_5^+\right|=p_2$, and therefore
\[
Z_5=\left|p_3+p_4-p_1\right|+p_2-\left|p_2-p_4\right|-\left|p_1-p_3\right|.
\]

\emph{The diagonal gates at the vertex $2$.}
We show
\begin{equation}\label{eq:hexagon-z2}
Z_2=\left|p_2+p_4-p_1\right|-\left|p_1+p_4-p_2\right|.
\end{equation}
With $\mu_0=\mu_5=M$ the four forms are
\[
\begin{aligned}
\Phi_2^+&=\delta_2-\mu_2-\mu_3+M, &\qquad \Phi_2^-&=\delta_2-\mu_1-M+\mu_4,\\
\Psi_2^+&=\delta_2-\mu_1-\mu_3+M, &\qquad \Psi_2^-&=\delta_2-\mu_2-M+\mu_4.
\end{aligned}
\]
We apply \Cref{lem:hexagon-diagonal} to the triangle $x_1,x_2,x_3$ and treat its three cases.
If $x_2$ is largest among $x_1,x_2,x_3$, then $\mu_1=\mu_2$, so $Z_2=0$.
Also $p_1=p_2$ in this case, so the right side of \Cref{eq:hexagon-z2} is $\left|p_4\right|-\left|p_4\right|=0$.
If $x_1$ is largest, then $\delta_2=\mu_1=x_1$ and $\mu_1\ge\mu_2$, that is, $p_2\ge p_1$.
Substituting $\delta_2=\mu_1=M-p_1$ gives
\[
\Phi_2^+=p_2-p_1+p_3,\qquad \Phi_2^-=-p_4,\qquad \Psi_2^+=p_3,\qquad \Psi_2^-=p_2-p_1-p_4.
\]
Here $\Phi_2^+\ge0$ and $\Psi_2^+\ge0$ because $p_2\ge p_1$ and $p_3\ge0$, and $\left|\Phi_2^-\right|=p_4$ because $p_4\ge0$.
So
\[
\begin{aligned}
Z_2&={(p_2-p_1+p_3)}+p_4-p_3-\left|p_2-p_1-p_4\right|\\
&={(p_2+p_4-p_1)}-\left|p_1+p_4-p_2\right|.
\end{aligned}
\]
Since $p_2\ge p_1$ and $p_4\ge0$, the number $p_2+p_4-p_1$ is nonnegative and equals its absolute value, which gives \Cref{eq:hexagon-z2}.
If $x_3$ is largest, then $\delta_2=\mu_2=x_3$ and $\mu_2\ge\mu_1$, that is, $p_1\ge p_2$.
Substituting $\delta_2=\mu_2=M-p_2$ gives
\[
\Phi_2^+=p_3,\qquad \Phi_2^-=p_1-p_2-p_4,\qquad \Psi_2^+=p_1-p_2+p_3,\qquad \Psi_2^-=-p_4.
\]
Here $\Phi_2^+\ge0$ and $\Psi_2^+\ge0$ because $p_1\ge p_2$ and $p_3\ge0$, and $\left|\Psi_2^-\right|=p_4$.
So $Z_2=p_3+\left|p_1-p_2-p_4\right|-{(p_1-p_2+p_3)}-p_4=\left|p_2+p_4-p_1\right|-{(p_1+p_4-p_2)}$.
Since $p_1\ge p_2$ and $p_4\ge0$, the number $p_1+p_4-p_2$ is nonnegative and equals its absolute value, which again gives \Cref{eq:hexagon-z2}.

\emph{The diagonal gates at the vertex $4$.}
We show
\begin{equation}\label{eq:hexagon-z4}
Z_4=\left|p_1+p_3-p_4\right|-\left|p_1+p_4-p_3\right|.
\end{equation}
With $\mu_0=\mu_5=M$ the four forms are
\[
\begin{aligned}
\Phi_4^+&=\delta_4-\mu_4-M+\mu_1, &\qquad \Phi_4^-&=\delta_4-\mu_3-\mu_2+M,\\
\Psi_4^+&=\delta_4-\mu_3-M+\mu_1, &\qquad \Psi_4^-&=\delta_4-\mu_4-\mu_2+M.
\end{aligned}
\]
We apply \Cref{lem:hexagon-diagonal} to the triangle formed by $x_3,x_4,x_5$.
If $x_4$ is largest among $x_3,x_4,x_5$, then $\mu_3=\mu_4$, so $Z_4=0$.
Also $p_3=p_4$ in this case, so the right side of \Cref{eq:hexagon-z4} is $\left|p_1\right|-\left|p_1\right|=0$.
If $x_5$ is largest, then $\delta_4=\mu_4=x_5$ and $\mu_4\ge\mu_3$, that is, $p_3\ge p_4$.
Substituting $\delta_4=\mu_4=M-p_4$ gives
\[
\Phi_4^+=-p_1,\qquad \Phi_4^-=p_3-p_4+p_2,\qquad \Psi_4^+=p_3-p_4-p_1,\qquad \Psi_4^-=p_2.
\]
Here $\left|\Phi_4^+\right|=p_1$, and $\Phi_4^-\ge0$ and $\Psi_4^-\ge0$ because $p_3\ge p_4$ and $p_2\ge0$.
So $Z_4=p_1+{(p_3-p_4+p_2)}-\left|p_3-p_4-p_1\right|-p_2={(p_1+p_3-p_4)}-\left|p_1+p_4-p_3\right|$.
Since $p_3\ge p_4$ and $p_1\ge0$, the number $p_1+p_3-p_4$ is nonnegative, which gives \Cref{eq:hexagon-z4}.
If $x_3$ is largest, then $\delta_4=\mu_3=x_3$ and $\mu_3\ge\mu_4$, that is, $p_4\ge p_3$.
Substituting $\delta_4=\mu_3=M-p_3$ gives
\[
\Phi_4^+=p_4-p_3-p_1,\qquad \Phi_4^-=p_2,\qquad \Psi_4^+=-p_1,\qquad \Psi_4^-=p_4-p_3+p_2.
\]
Here $\Phi_4^-\ge0$, $\left|\Psi_4^+\right|=p_1$, and $\Psi_4^-\ge0$ because $p_4\ge p_3$ and $p_2\ge0$.
So
\[
\begin{aligned}
Z_4&=\left|p_4-p_3-p_1\right|+p_2-p_1-{(p_4-p_3+p_2)}\\
&=\left|p_1+p_3-p_4\right|-{(p_1+p_4-p_3)}.
\end{aligned}
\]
Since $p_4\ge p_3$ and $p_1\ge0$, the number $p_1+p_4-p_3$ is nonnegative, which gives \Cref{eq:hexagon-z4}.

\emph{The diagonal gates at the vertex $3$.}
We show $Z_3=0$.
With $\mu_0=\mu_5=M$ the four forms are
\[
\begin{aligned}
\Phi_3^+&=\delta_3-\mu_3-\mu_4+M, &\qquad \Phi_3^-&=\delta_3-\mu_2-\mu_1+M,\\
\Psi_3^+&=\delta_3-\mu_2-\mu_4+M, &\qquad \Psi_3^-&=\delta_3-\mu_3-\mu_1+M.
\end{aligned}
\]
We apply \Cref{lem:hexagon-diagonal} to the triangle formed by $x_2,x_3,x_4$.
If $x_3$ is largest among $x_2,x_3,x_4$, then $\mu_2=\mu_3$ and $Z_3=0$.
If $x_2$ is largest, then $\delta_3=\mu_2=x_2$ and $\mu_2\ge\mu_3$.
The four forms become
\[
\Phi_3^+=(\mu_2-\mu_3)+p_4,\qquad \Phi_3^-=p_1,\qquad \Psi_3^+=p_4,\qquad \Psi_3^-=(\mu_2-\mu_3)+p_1,
\]
and all four are nonnegative because $\mu_2\ge\mu_3$ and $p_1,p_4\ge0$.
Hence $Z_3=\Phi_3^++\Phi_3^--\Psi_3^+-\Psi_3^-=0$.
If $x_4$ is largest, then $\delta_3=\mu_3=x_4$ and $\mu_3\ge\mu_2$.
The four forms become
\[
\Phi_3^+=p_4,\qquad \Phi_3^-=(\mu_3-\mu_2)+p_1,\qquad \Psi_3^+=(\mu_3-\mu_2)+p_4,\qquad \Psi_3^-=p_1,
\]
again all nonnegative, and again $Z_3=\Phi_3^++\Phi_3^--\Psi_3^+-\Psi_3^-=0$.

\emph{Cancellation.}
Three pairs of aggregates now combine: $Y_0$ with $Z_2$, then $Y_2$ with $Z_4$, and finally $Z_1$ with $Z_5$.
First add \Cref{eq:hexagon-z2} to the value of $Y_0$.
The terms $\left|p_1+p_4-p_2\right|$ and $\left|p_2+p_4-p_1\right|$ cancel, and
\[
Y_0+Z_2=-\left|p_1+p_2-p_4\right|.
\]
In the same way, add \Cref{eq:hexagon-z4} to the value of $Y_2$.
The terms $\left|p_1+p_4-p_3\right|$ and $\left|p_1+p_3-p_4\right|$ cancel, and
\[
Y_2+Z_4=-\left|p_3+p_4-p_1\right|.
\]
Adding the values of $Z_1$ and $Z_5$ gives
\[
Z_1+Z_5=\left|p_1+p_2-p_4\right|+\left|p_3+p_4-p_1\right|+p_2+p_3-2\left|p_2-p_4\right|-2\left|p_1-p_3\right|.
\]
Together with $Y_1=-(p_2+p_3)$ and $Z_0=Z_3=0$, the sum of the aggregates is
\[
\sum_{k=0}^{2}Y_k+\sum_{v=0}^{5}Z_v=-\left|p_1+p_2-p_4\right|-\left|p_3+p_4-p_1\right|+\left(Z_1+Z_5\right)-\left(p_2+p_3\right).
\]
Substituting the value of $Z_1+Z_5$, the remaining terms cancel and
\[
\sum_{k=0}^{2}Y_k+\sum_{v=0}^{5}Z_v=-2\left|p_1-p_3\right|-2\left|p_2-p_4\right|.
\]
Adding \Cref{eq:hexagon-theta} gives exactly the right side of \Cref{eq:hexagon-target}.
This proves \Cref{eq:hexagon-abs} whenever $x_0$ is a largest coordinate, and hence for every $x\in\R^6$ by the reduction.

\emph{The ReLU form.}
Write $\mathcal F^+$ and $\mathcal F^-$ for the lists of positive and negative gate forms.
By definition the right side of \Cref{eq:hexagon-abs} is $2\sum_i\mu_i+\sum_{F\in\mathcal F^+}\left|F\right|-\sum_{F\in\mathcal F^-}\left|F\right|$.
Substitute $\left|t\right|=2\ReLU(t)-t$ into every absolute value.
The linear terms that this produces are $-\Lambda$, where
\[
\begin{split}
\Lambda&=\sum_{F\in\mathcal F^+}F-\sum_{F\in\mathcal F^-}F\\
&=\sum_{j=0}^{5}\Theta_j+\sum_{k=0}^{2}\left(\Omega_k-\Pi_k-\Gamma_k\right)+\sum_{v=0}^{5}\left(\Phi_v^++\Phi_v^--\Psi_v^+-\Psi_v^-\right).
\end{split}
\]
We compute the three sums.
First, $\sum_j\Theta_j=2\sum_j\mu_j-2\sum_j\mu_{j+2}=0$.
Second, by the formulas in \Cref{def:hexagon-forms},
\[
\Omega_k-\Pi_k-\Gamma_k=\mu_{k+2}+\mu_{k+4}+\mu_{k+5}-3\mu_{k+1},
\]
since the edge $\mu_{k+1}$ enters each of $\Omega_k$, $-\Pi_k$ and $-\Gamma_k$ with the coefficient $-1$.
Each of the other three edges enters two of these forms with $+1$ and one with $-1$.
Summing over $k=0,1,2$, each of the edges $\mu_1$, $\mu_2$ and $\mu_3$ occurs once with $-3$ and once with $+1$.
Each of the edges $\mu_0$, $\mu_4$ and $\mu_5$ occurs twice with $+1$: the edge $\mu_0$ for $k=1,2$, the edge $\mu_4$ for $k=0,2$, and the edge $\mu_5$ for $k=0,1$.
Hence
\[
\sum_{k=0}^{2}\left(\Omega_k-\Pi_k-\Gamma_k\right)=2\mu_0-2\mu_1-2\mu_2-2\mu_3+2\mu_4+2\mu_5.
\]
Third, by the first part of \Cref{lem:hexagon-diagonal}, $(\Phi_v^+-\Psi_v^+)+(\Phi_v^--\Psi_v^-)={(\mu_{v-1}-\mu_v)}-{(\mu_{v-1}-\mu_v)}=0$ for every $v$.
Therefore $\Lambda=2(\mu_0-\mu_1-\mu_2-\mu_3+\mu_4+\mu_5)$, and \Cref{eq:hexagon-abs} becomes
\[
12\Max_6=2\sum_{i=0}^{5}\mu_i-\Lambda+2\sum_{F\in\mathcal F^+}\ReLU(F)-2\sum_{F\in\mathcal F^-}\ReLU(F).
\]
Here $2\sum_i\mu_i-\Lambda=4(\mu_1+\mu_2+\mu_3)$.
Dividing by $12$ gives \Cref{eq:hexagon-relu}.
Each step of this rewriting is an equality between functions of $x$ and can be reversed, so the two displayed identities are equivalent.
\end{proof}

The reflections of the hexagon also map the right side of \Cref{eq:hexagon-abs} term by term onto itself, although the proof above uses rotations only.

\begin{remark}[Reflections]
\label{rem:hexagon-reflection}
Define $y\in\R^6$ by $y_i=x_{-i}$.
Then $\mu_i(y)=\mu_{-i-1}(x)$ and $\delta_v(y)=\delta_{-v}(x)$.
We substitute these values into \Cref{def:hexagon-forms}, with $\Pi_k$, $\Omega_k$ and $\Gamma_k$ defined for every $k$ modulo $6$ as in the proof of \Cref{lem:hexagon-rotation}.
For the distance-two differences, $\Theta_j(y)=2(\mu_{-j-1}-\mu_{-j-3})$ at $x$, which is $-\Theta_{-j-3}(x)$.
For the pairings, put $l=-k-1$.
Then $\Pi_k(y)=\mu_{-k-2}+\mu_{-k-3}-\mu_{-k-5}-\mu_{-k-6}$ at $x$, and the four indices are $l+5$, $l+4$, $l+2$ and $l+1$.
Hence $\Pi_k(y)=-\Pi_l(x)$.
Since $\mu_{k+1}$, $\mu_{k+2}$, $\mu_{k+4}$ and $\mu_{k+5}$ at $y$ are $\mu_{l+5}$, $\mu_{l+4}$, $\mu_{l+2}$ and $\mu_{l+1}$ at $x$, the same substitution gives
\[
\begin{aligned}
\Omega_k(y)&=\mu_{l+4}+\mu_{l+1}-\mu_{l+5}-\mu_{l+2}=-\Omega_l(x),\\
\Gamma_k(y)&=\mu_{l+5}+\mu_{l+1}-\mu_{l+4}-\mu_{l+2}=\Gamma_l(x).
\end{aligned}
\]
For the diagonal gates, put $w=-v$.
Then $\delta_v(y)=\delta_w(x)$ and $\mu_{v+c}(y)=\mu_{w-c-1}(x)$ for every $c$, so
\[
\begin{aligned}
\Phi_v^+(y)&=\delta_w-\mu_{w-1}-\mu_{w-2}+\mu_{w+2}=\Phi_w^-(x),\\
\Psi_v^+(y)&=\delta_w-\mu_w-\mu_{w-2}+\mu_{w+2}=\Psi_w^-(x),\\
\Phi_v^-(y)&=\delta_w-\mu_w-\mu_{w+1}+\mu_{w+3}=\Phi_w^+(x),\\
\Psi_v^-(y)&=\delta_w-\mu_{w-1}-\mu_{w+1}+\mu_{w+3}=\Psi_w^+(x),
\end{aligned}
\]
where the middle expressions are evaluated at $x$.

As $k$ runs over $0,1,2$, the index $l$ runs over $5,4,3$, and $\Pi_l=-\Pi_{l-3}$, $\Omega_l=\Omega_{l-3}$ and $\Gamma_l=-\Gamma_{l-3}$ with $l-3=2-k$.
So $\Pi_k(y)=\Pi_{2-k}(x)$, $\Omega_k(y)=-\Omega_{2-k}(x)$ and $\Gamma_k(y)=-\Gamma_{2-k}(x)$.
The maps $j\mapsto-j-3$ and $v\mapsto-v$ on residues modulo $6$ and $k\mapsto2-k$ on $\{0,1,2\}$ are bijections.
Let $\sigma$ map $\Theta_j$ to $\Theta_{-j-3}$, the pairings $\Pi_k$, $\Omega_k$, $\Gamma_k$ to $\Pi_{2-k}$, $\Omega_{2-k}$, $\Gamma_{2-k}$, and $\Phi_v^\pm$, $\Psi_v^\pm$ to $\Phi_{-v}^\mp$, $\Psi_{-v}^\mp$.
By these three bijections, $\sigma$ is a bijection of the $39$ forms.
It maps positive gate forms to positive gate forms and negative gate forms to negative gate forms, since $\Phi_v^+$ and $\Phi_v^-$ are positive and $\Psi_v^+$ and $\Psi_v^-$ are negative.
By the equalities above, for each form $F$ either $F(y)=\sigma(F)(x)$ for every $x$ or $F(y)=-\sigma(F)(x)$ for every $x$.
Also $\sum_i\mu_i(y)=\sum_i\mu_{-i-1}(x)=\sum_i\mu_i(x)$, since $i\mapsto-i-1$ permutes the residues modulo $6$.
So the reflection has the properties of $\sigma$ and of $\sum_i\mu_i$ stated in \Cref{lem:hexagon-rotation}.

Now let $y$ be the image of $x$ under any reflection of the hexagon, so that $y_i=x_{t-i}$ for some $t\in\{0,\ldots,5\}$.
This reflection is the one treated above, which maps $x$ to $z$ with $z_i=x_{-i}$, followed by a rotation.
Indeed, $y_i=x_{t-i}=z_{i-t}=z_{i+t'}$, where $t'\in\{0,\ldots,5\}$ and $t'\equiv-t$ modulo $6$, and this is the rotation of \Cref{lem:hexagon-rotation} with the shift $t'$, applied to $z$.

These properties survive composition.
If $x\mapsto z$ has them with a bijection $\sigma_1$ and $z\mapsto y$ with a bijection $\sigma_2$, then for each form $F$ we have $F(y)=\pm\sigma_2(F)(z)$ and $\sigma_2(F)(z)=\pm\sigma_1(\sigma_2(F))(x)$, each sign the same for every $x$, so $F(y)=\pm\sigma_1(\sigma_2(F))(x)$ with the product of the two signs.
The composite $\sigma_1\circ\sigma_2$ is a bijection that keeps the positive and the negative gate forms, and $\sum_i\mu_i(y)=\sum_i\mu_i(z)=\sum_i\mu_i(x)$.
Hence all twelve symmetries of the hexagon have these properties.

Each of them maps the right side of \Cref{eq:hexagon-abs} term by term onto itself, in the sense that the absolute values $\left|F\right|$ of the positive gate forms are permuted among themselves, likewise for the negative ones, and $\sum_i\mu_i$ keeps its value.
By the definitions of $S_\Theta$, $Y_k$ and $Z_v$, the right side is $2\sum_i\mu_i$ plus the sum of $\left|F\right|$ over the positive gate forms $F$ minus the sum of $\left|F\right|$ over the negative ones.
Now let $y$ be the image of $x$ under any one of the twelve symmetries, and let $\sigma$ be its bijection of the forms.
At $y$ the term $\left|F(y)\right|$ equals the term $\left|\sigma(F)(x)\right|$ at $x$, and as $F$ runs over the positive gate forms, so does $\sigma(F)$, and likewise for the negative ones.
So $\sum_{F\text{ positive}}\left|F(y)\right|=\sum_{F\text{ positive}}\left|\sigma(F)(x)\right|=\sum_{F\text{ positive}}\left|F(x)\right|$, the same holds for the negative gate forms, and $\sum_i\mu_i(y)=\sum_i\mu_i(x)$.
\end{remark}

\section{A Network with Widths Seventeen and Forty-One}
\label{sec:network}

\Cref{eq:hexagon-relu} uses only the twelve pairwise maxima $\mu_i$ and $\delta_v$.
We now turn the identity into a fully specified network, including its gate counts and weights.
The first layer below has seventeen gates, and each of the twelve maxima is a linear form in their outputs.
For an ordered pair $(a,b)$ write $r_{ab}(x)=\ReLU(x_a-x_b)$ for the oriented comparison gate from $a$ to $b$.
Let $\mathcal E$ be the list of the $15$ ordered pairs
\[
\begin{aligned}
&(0,1),\quad(0,2),\quad(0,4),\quad(0,5),\quad(3,1),\quad(3,2),\quad(3,4),\quad(3,5),\quad(4,5),\\
&(1,2),\quad(2,1),\quad(1,5),\quad(5,1),\quad(2,4),\quad(4,2).
\end{aligned}
\]
In words, each of the twelve pairs $\{a,b\}$ of a hexagon edge or short diagonal has at least one oriented comparison gate.
The eight pairs containing $0$ or $3$ are oriented away from $0$ or $3$, and the pair $\{4,5\}$ is oriented from $4$ to $5$.
The three pairs $\{1,2\}$, $\{1,5\}$ and $\{2,4\}$ carry both orientations; we call them the doubled pairs.
The first hidden layer has the $15$ outputs $r_{ab}$ for $(a,b)\in\mathcal E$, together with $q_+(x)=\ReLU(x_1)$ and $q_-(x)=\ReLU(-x_1)$.
Write $\widehat R(x)$ for this $17$-component vector.

For each of the twelve pairs choose one oriented gate: the unique one where there is only one, and $(1,2)$, $(1,5)$, $(2,4)$ for the doubled pairs.
In every case the chosen gate $r_{ab}$ has its second index $b$ in the set $\{1,2,4,5\}$, and
\begin{equation}\label{eq:hexagon-low}
m_{ab}=\max(x_a,x_b)=x_b+\max(x_a-x_b,0)=x_b+r_{ab}.
\end{equation}
The coordinates $x_1,x_2,x_4,x_5$ are linear combinations of the outputs.
Indeed, put $w_{ab}=x_a-x_b=r_{ab}-r_{ba}$ for each doubled pair.
Then
\begin{equation}\label{eq:hexagon-coords}
\begin{split}
&x_1=q_+-q_-,\qquad x_2=x_1-w_{12},\qquad x_5=x_1-w_{15},\\
&x_4=x_2-w_{24}=x_1-w_{12}-w_{24}.
\end{split}
\end{equation}
We need not reconstruct $x_0$ or $x_3$ from the first-layer outputs.

\begin{lemma}[Linear parts]
\label{lem:hexagon-linear}
Let $F=\sum c_{ab}m_{ab}$ be a combination of the twelve maxima $\mu_i,\delta_v$ with rational coefficients $c_{ab}$.
Here each pair is written as $ab$ in the orientation of its chosen gate.
Let $\kappa_\ell$ be the sum of the $c_{ab}$ over the pairs whose chosen gate has second index $\ell$.
Then
\[
F=\sum c_{ab}r_{ab}+\left(\kappa_1+\kappa_2+\kappa_4+\kappa_5\right)x_1-\left(\kappa_2+\kappa_4\right)w_{12}-\kappa_4w_{24}-\kappa_5w_{15},
\]
where each $r_{ab}$ is the chosen gate of its pair.
In particular, if the coefficients $c_{ab}$ sum to zero, then $F$ equals $\sum c_{ab}r_{ab}$ plus a combination of $w_{12}$, $w_{24}$ and $w_{15}$ with coefficients built from the $\kappa_\ell$.
\end{lemma}

\begin{proof}
By \Cref{eq:hexagon-low}, $F=\sum c_{ab}r_{ab}+\kappa_1x_1+\kappa_2x_2+\kappa_4x_4+\kappa_5x_5$.
Substituting \Cref{eq:hexagon-coords} gives
\[
\kappa_1x_1+\kappa_2\left(x_1-w_{12}\right)+\kappa_4\left(x_1-w_{12}-w_{24}\right)+\kappa_5\left(x_1-w_{15}\right),
\]
which is the stated expression.
The sum $\kappa_1+\kappa_2+\kappa_4+\kappa_5$ is the sum of all coefficients $c_{ab}$, which gives the last claim.\qedhere
\end{proof}

By this lemma, every linear combination of the twelve maxima, and in particular every argument of a ReLU in \Cref{eq:hexagon-relu}, is a linear form in the seventeen outputs.

\begin{corollary}[restatement of \Cref{icor:hexagon-network}]
\label{cor:hexagon-network}
There is a zero-bias ReLU network with architecture $6\to17\to41\to1$ computing $\Max_6$ on all of $\R^6$.
Every weight is an integer multiple of $1/6$.
\end{corollary}

\begin{proof}
The first hidden layer is $\widehat R(x)$.
Its weights are $0$ and $\pm1$.

Each of the $39$ hexagon gate forms $F$ is a combination of the twelve maxima with integer coefficients summing to zero, as noted after \Cref{def:hexagon-forms}.
By \Cref{lem:hexagon-linear}, $F(x)=\widehat F(\widehat R(x))$ for the linear form
\[
\widehat F=\sum c_{ab}r_{ab}-\left(\kappa_2+\kappa_4\right)\left(r_{12}-r_{21}\right)-\kappa_4\left(r_{24}-r_{42}\right)-\kappa_5\left(r_{15}-r_{51}\right)
\]
on the outputs.
Its coefficients are integers, because the $c_{ab}$ and hence the $\kappa_\ell$ are integers.

The linear term of \Cref{eq:hexagon-relu} is $L={\frac13(\mu_1+\mu_2+\mu_3)}$, which equals ${\frac13(m_{12}+m_{23}+m_{34})}$.
Its coefficients sum to $1$.
The chosen gates of its three pairs are $r_{12}$, $r_{32}$ and $r_{34}$, with second indices $2$, $2$ and $4$.
So $\kappa_2=\frac23$, $\kappa_4=\frac13$ and $\kappa_1=\kappa_5=0$, and \Cref{lem:hexagon-linear} gives
\[
L=\frac13\left(r_{12}+r_{32}+r_{34}\right)+x_1-w_{12}-\frac13w_{24}.
\]
Writing $x_1=q_+-q_-$, $w_{12}=r_{12}-r_{21}$ and $w_{24}=r_{24}-r_{42}$, the corresponding form on the outputs is
\[
\widehat L=q_+-q_--\frac23r_{12}+r_{21}+\frac13r_{32}+\frac13r_{34}-\frac13r_{24}+\frac13r_{42}.
\]
Thus $L(x)=\widehat L(\widehat R(x))$, and the coefficients of $\widehat L$ lie in $\frac13\Z$.

The second hidden layer has the $39$ gates $\ReLU(\widehat F(\widehat R))$, one for each hexagon gate form, and the two gates $\ReLU(\widehat L(\widehat R))$ and $\ReLU(-\widehat L(\widehat R))$.
Its width is $39+2=41$.
The readout gives weight $\frac16$ to each gate of a positive form, $-\frac16$ to each gate of a negative form, and weights $1$ and $-1$ to the last two gates.
Since $\ReLU(t)-\ReLU(-t)=t$, the last two gates contribute $L(x)$.
The output is therefore the right side of \Cref{eq:hexagon-relu}, which equals $\Max_6(x)$ by \Cref{thm:hexagon}.
The first-layer weights are integers, the second-layer weights lie in $\Z\cup\frac13\Z=\frac13\Z$, and the readout weights lie in $\{\pm\frac16,\pm1\}$.
All of them are integer multiples of $\frac16$, and every bias is zero.\qedhere
\end{proof}

The same network computes the maximum of fewer numbers.
For $1\le n\le6$, write $P$ for the $6\times n$ matrix that repeats the last coordinate:
\[
P(x_0,\ldots,x_{n-1})=(x_0,\ldots,x_{n-1},x_{n-1},\ldots,x_{n-1}).
\]
Then $\Max_6(Px)=\Max_n(x)$, so replacing the first weight matrix $W_1$ by $W_1P$ gives a zero-bias network $n\to17\to41\to1$ computing $\Max_n$.
The entries of $W_1P$ are sums of entries of $W_1$, so they are again integers.
We claim no optimality for either width of these networks.

\phantomsection\addcontentsline{toc}{section}{References}

\newcommand{\etalchar}[1]{$^{#1}$}

\end{document}